\documentclass[sigconf]{acmart}
\AtBeginDocument{%
  }

\copyrightyear{2023}
\acmYear{2023}
\setcopyright{rightsretained}
\acmConference[KDD '23] {Proceedings of the 29th ACM SIGKDD Conference on Knowledge Discovery and Data Mining}{August 6--10, 2023}{Long Beach, CA, USA.}
\acmBooktitle{Proceedings of the 29th ACM SIGKDD Conference on Knowledge Discovery and Data Mining (KDD '23), August 6--10, 2023, Long Beach, CA, USA}

\acmPrice{}
\acmISBN{979-8-4007-0103-0/23/08}
\acmDOI{10.1145/3580305.3599300}

\newcommand{\out}[1]{}
\usepackage{algorithm}
\usepackage{algorithmic}

\usepackage{amsmath,amsfonts,bm}

\def\eqref#1{equation~\ref{#1}}

\def\1{\bm{1}}

\def\mA{{\bm{A}}}

\def\mI{{\bm{I}}}

\DeclareMathAlphabet{\mathsfit}{\encodingdefault}{\sfdefault}{m}{sl}
\SetMathAlphabet{\mathsfit}{bold}{\encodingdefault}{\sfdefault}{bx}{n}

\def\gS{{\mathcal{S}}}

\usepackage{url}

\usepackage{microtype}
\usepackage{graphicx}
\usepackage{booktabs} %
\usepackage{pifont}
\usepackage{hyperref}

\usepackage{caption}
\usepackage{natbib}

\usepackage{amsmath,amssymb,amsfonts,amsthm}

\usepackage{xcolor}
\usepackage{mathtools}
\usepackage{dsfont}
\usepackage{hyperref}
\usepackage{bm}
\usepackage{nicefrac}
\usepackage{wrapfig}
\usepackage{lipsum}

\usepackage{algorithm,algorithmic}
\usepackage[linesnumbered, ruled,vlined,algo2e]{algorithm2e}

\newcounter{assumption}%
\renewcommand{\theassumption}{\arabic{assumption}}

\def\rT{{\rm T}}

\def\mI{\mathrm{I}}

\def\mI{\mathrm{I}}

\newcommand*\E[1]{\mathbb{E}\left[#1\right]}
\newcommand*\Ep[2]{\mathbb{E}_{#1}\left[#2\right]}

\newcommand*\lrb[1]{\left[#1\right]}
\newcommand*\lrbb[1]{\left\{#1\right\}}
\newcommand*\lrp[1]{\left(#1\right)}
\newcommand*\lrn[1]{\left\|#1\right\|}
\newcommand*\lrw[1]{\left\langle#1\right\rangle}

\newcommand{\rE}{{\mathbb E}}
\def\rT{{\rm T}}
\def\mI{\mathrm{I}}
\def\R{\mathbb{R}}

\def\E{e}
\def\mA{\mathrm{A}}

\def\mI{\mathrm{I}}

\def\rT{{\rm T}}

\def\mI{\mathrm{I}}

\newtheorem{lemma}{Lemma}
\newtheorem{theorem}{Theorem}

\newtheorem{proposition}{Proposition}

\DeclareMathOperator*{\argmax}{argmax}

\usepackage{multirow}
\usepackage{caption}
\usepackage{subcaption}

\usepackage{booktabs}
\usepackage{amssymb,amsmath}

\newcommand{\acf}{\textsc{LIG}}

\newcommand{\ours}{\textsc{INP}}

\begin{document}

\title{Deep Bayesian Active Learning for \\Accelerating Stochastic Simulation}

\author{Dongxia Wu}
\affiliation{%
  \institution{University of California, San Diego}
  \city{La Jolla}
  \state{CA}
  \country{USA}
}
\email{dowu@ucsd.edu}

\author{Ruijia Niu}
\affiliation{%
  \institution{University of California, San Diego}
  \city{La Jolla}
  \state{CA}
  \country{USA}
}
\email{rniu@ucsd.edu}

\author{Matteo Chinazzi}
\affiliation{%
 \institution{Northeastern University}
 \city{Boston}
 \state{MA}
 \country{USA}
 }
 \email{m.chinazzi@northeastern.edu}

\author{Alessandro Vespignani}
\affiliation{%
  \institution{Northeastern University}
  \city{Boston}
  \state{MA}
  \country{USA}
 }
 \email{a.vespignani@northeastern.edu}

\author{Yi-An Ma}
\affiliation{
  \institution{University of California, San Diego}
  \city{La Jolla}
  \state{CA}
  \country{USA}
}
\email{yianma@ucsd.edu}

\author{Rose Yu}
\affiliation{
  \institution{University of California, San Diego}
  \city{La Jolla}
  \state{CA}
  \country{USA}
}
\email{roseyu@ucsd.edu}
\renewcommand{\shortauthors}{Dongxia Wu et al.}

\begin{abstract}


Stochastic simulations such as large-scale, spatiotemporal, age-structured epidemic models are computationally expensive at fine-grained resolution. While deep surrogate models can speed up the simulations, doing so for stochastic simulations and with active learning approaches is an underexplored area. We propose Interactive Neural Process (\ours{}), a deep Bayesian active learning framework for learning deep surrogate models to accelerate stochastic simulations. \ours{} consists of two components, a spatiotemporal surrogate model built upon Neural Process (NP) family and an acquisition function for active learning.
For surrogate modeling, we develop Spatiotemporal Neural Process (STNP) to mimic the simulator dynamics. For active learning, we propose a novel acquisition function, Latent Information Gain (LIG), calculated in the latent space of NP based models. We perform a theoretical analysis and demonstrate that LIG reduces sample complexity compared with random sampling in high dimensions. We also conduct empirical studies on three complex spatiotemporal simulators for reaction diffusion, heat flow, and infectious disease. The results demonstrate that STNP outperforms the baselines in the offline learning setting and LIG achieves the state-of-the-art for Bayesian active learning. 
\end{abstract}



\maketitle

\section{Introduction}
Computational modeling is more than ever at the forefront of infectious disease research due to the COVID-19 pandemic. Stochastic simulations play a critical role in understanding and forecasting infectious disease dynamics, creating what-if scenarios, and informing public health policy making \citep{cramer2021evaluation}.  More broadly,  stochastic simulations \citep{ripley2009stochastic,asmussen2007stochastic} produce forecasts about complex interactions among people, environment, space, and time given a set of parameters. They provide the numerical tools to simulate  
stochastic processes in finance \citep{lamberton2007introduction}, chemistry \citep{gillespie2007stochastic} and many other scientific disciplines.

Unfortunately, stochastic simulations at fine-grained spatial and temporal resolution can be extremely computationally expensive. In example, epidemic models for realistic diffusion dynamics simulation via in-silico experiments require a large parameter space (e.g. characteristics of a virus, policy interventions, people's behavior). Similarly, reaction-diffusion systems that play an important role in chemical reaction and bio-molecular processes also involve a large number of simulation conditions. Therefore, hundreds of thousands of simulations are required to explore and calibrate the simulation model with observed experimental data. This process significantly hinders the adaptive capability of existing stochastic simulators, especially in ``war time'' emergencies, due to the lead time needed to execute new simulations and produce actionable insights that could help guide decision makers.

Learning deep surrogate models to speed up complex simulation has been explored in climate modeling and fluid dynamics for \textit{deterministic} dynamics \citep{sanchez2020learning, wang2020towards,holl2019learning,rasp2018deep, cachay2021climart},  but not for \textit{stochastic} simulations. These surrogate models can only approximate specific system dynamics and fail to generalize under different parametrization. Especially for pandemic  scenario planning, we desire models that can predict futuristic scenarios under different  conditions.
Furthermore, the majority of the surrogate models are trained \textit{passively} using  a simulation data set.  This requires a large number of  simulations beforehands to cover different parameter regimes of the simulator and ensure generalization.


We propose Interactive Neural Process (\ours{}), a deep Bayesian active learning framework to speed up stochastic simulations. Given parameters such as disease reproduction number, incubation and infectious periods, 
mechanistic simulators generate future outbreak states with time-consuming numerical integration. \ours{} accelerates the simulation by guiding a surrogate model to learn the input-output map between parameters and future states, hence bypassing numerical integration. 

The deep surrogate model of \ours{} is built upon Neural Process (NP) \cite{garnelo2018neural}, which lies between Gaussian process (GP) and neural network (NN). NPs can approximate stochastic processes and therefore are well-suitable for surrogate modeling of stochastic simulators. They learn distributions over functions and can generate prediction uncertainty for Bayesian active learning. Compared with GPs, NPs are more flexible and scalable for high-dimensional data with spatiotemporal dependencies. We design a novel Spatiotemporal Neural Process model (STNP) by introducing a time-evolving latent process for temporal dynamics and integrating spatial convolution for spatial modeling.


Instead of learning passively, we design \textit{active learning} algorithms to interact with the simulator and update our model in ``real-time’’. We derive a new acquisition function, Latent Information Gain (LIG), based on our unique model design.  Our algorithm selects the parameters with the highest LIG,  queries the simulator  to generate new simulation data, and continuously updates our model. We provide theoretical guarantees for the sample efficiency of this procedure over random sampling. We also demonstrate the efficacy of our method on  large-scale spatiotemporal epidemic and reaction diffusion models.  In summary, our contributions include:
\begin{itemize}
    \item Interactive Neural Process: a deep Bayesian active learning framework for accelerating large-scale stochastic simulation.
    \item New surrogate model, Spatiotemporal Neural Process (STNP), for high-dimensional spatiotemporal data that integrates temporal latent process and spatial convolution.
    \item New acquisition function, Latent Information Gain (\acf{}), based on the inferred  temporal latent process to quantify uncertainty with theoretical guarantees.
    \item Real-world application to speed up complex stochastic spatiotemporal  simulations including reaction-diffusion system, heat flow, and age-structured epidemic dynamics.
\end{itemize}

\section{Related Work}
\textbf{Bayesian Active Learning and Experimental Design.}
Bayesian active learning, or experimental design is well-studied in statistics and machine learning \citep{chaloner1995bayesian, cohn1996active}. Gaussian Processes (GPs) are popular for posterior estimation e.g. \cite{houlsby2011bayesian} and \citep{zimmer2018safe}, but often struggle in high-dimension. Deep neural networks provide scalable solutions  for active learning. Deep active learning has been applied to discrete problems such as image classification \citep{gal2017deep} and sequence labeling \citep{siddhant2018deep} whereas our task is continuous time series.   Our problem can also be viewed as sequential experimental design where we design simulation parameters to obtain the desired outcome (imitating the simulator). \citet{foster2021deep} propose deep design networks for Bayesian experiment design but they require a explicit likelihood model and conditional independence in experiments. \citet{kleinegesse2020bayesian} consider implicit models where the likelihood function is intractable, but computing the Jacobian through sampling path can be expensive and their experiments are mostly limited to low (<=10) dimensional design. In contrast, our design space is of much higher-dimension and we do not have access to an explicit likelihood model for the simulator.



\textbf{Neural Processes.}
Neural Processes (NP) \citep{garnelo2018neural} model distributions over functions and imbue  neural networks with the ability of GPs to estimate uncertainty. NP has many extensions such as attentive NP \citep{kim2019attentive} and functional NP \citep{louizos2019functional}. However, NP implicitly assumes permutation invariance in the latent variables and can be limiting in modeling temporal dynamics.    \citet{singh2019sequential} proposes sequential NP by incorporating a  temporal transition model into NP. Still, sequential NP assumes the latent variables are independent conditioned on the hidden states. We propose STNP with temporal latent process and spatial  convolution, which is well-suited for modeling the spatiotemporal dynamics of infectious disease. We apply our model to   real-world large-scale Bayesian active learning. Note that even though \citet{garnelo2018neural} has demonstrated  NP for Bayesian optimization, it is only for toy 1-D functions.

\textbf{Stochastic Simulation and Dynamics Modeling.}
Stochastic simulations are fundamental to many scientific fields \citep{ripley2009stochastic} such as epidemic modeling.  Data-driven models of infectious diseases are increasingly used to forecast the  evolution of an ongoing outbreak  \citep{arik2020interpretable, cramer2021evaluation,lourenco2020fundamental}. However, very few models can mimic the internal mechanism of a stochastic simulator and answer ``what-if questions''. GPs are commonly used   as surrogate models for expensive simulators \citep{meeds2014gps, gutmann2016bayesian, jarvenpaa2019efficient, qian2020and},  but GPs do not scale well to high-dimensional data. Likelihood-free inference methods \cite{lueckmann2019likelihood, papamakarios2019sequential, munk2019deep,wood2020planning} learn the posterior of the parameters given the observed data. They do neural density estimation, but require a lot of simulations. For active learning,  instead of relying on Monte Carlo sampling, we directly compute the information gain in the latent process. \citet{qian2020and} use GPs as a prior for a SEIR model  for learning lockdown policy effects, but GPs are computationally expensive and the simple SEIR model cannot capture the real-world large-scale, spatiotemporal dynamics considered in this work.  We demonstrate the use of deep sequence model as a prior distribution in Bayesian active learning. Our framework is also compatible with other  deep  sequence models for time series, e.g. Deep State Space \citep{rangapuram2018deep}, Neural ODE \citep{chen2018neural}. 



\section{Methodology}

\begin{figure*}[t]
  \centering
  \includegraphics[width=0.95\linewidth,trim={20 20 20 20}]{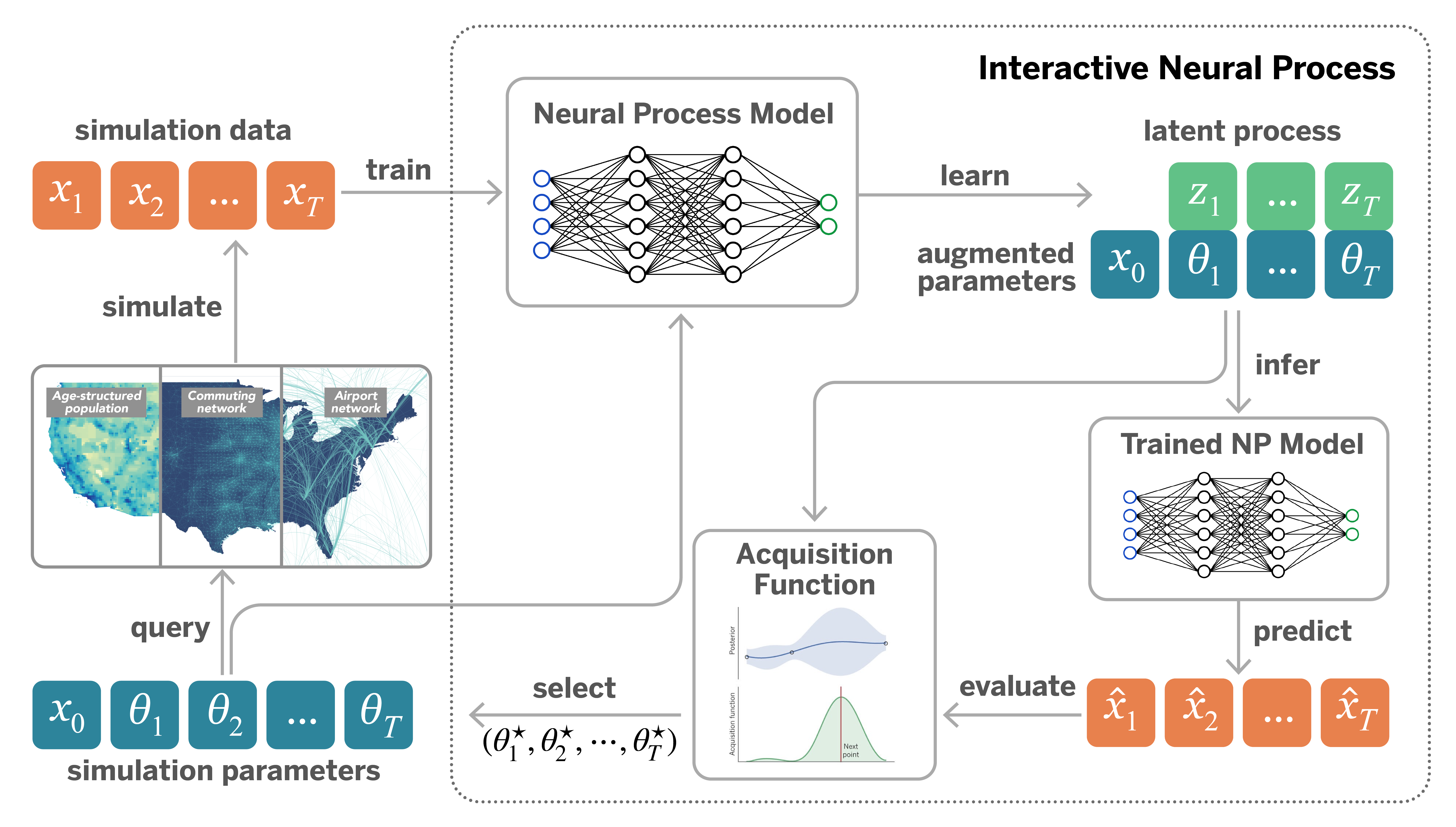}
  \caption{Illustration of the interactive Neural Process (\ours{}). Given simulation parameters and data, \ours{} trains a surrogate model (e.g. STNP) to infer the latent process. The inferred latent process allows prediction and uncertainty quantification. They are used to calculate the acquisition function (e.g. LIG) to select the next set of parameters to query, and simulate more data.}
  \label{fig:overview}
\end{figure*}

Consider a stochastic process $\{X_1, \cdots, X_T \}$, governed by  time-varying parameters $ \theta_t \in \mathbb{R}^K$, and the initial state  $x_0  \in \mathbb{R}^D$.   In epidemic modeling, $\theta_t$ can represent the effective reproduction number of the virus at a given time, the effective contact rates between individuals belonging to different age groups, the people's degree of short- or long-range mobility, or the effects of time varying policy interventions (e.g. non-pharmaceutical interventions). The state $x_t \in \mathbb{R}^D$ includes both the daily prevalence and daily incidence for each compartment of the epidemic model (e.g. number of people that are infectious and number of new infected individuals at time $t$).

Stochastic simulation uses a mechanistic model  $F( \theta ; \xi )$ to simulate the process where the random variable $\xi$ represents the randomness in the simulator. Let  $\theta :=(x_0, \theta_1, \cdots, \theta_T)$ represent the initial state and all the parameters over time. For each $\theta$, we obtain a different  set of simulation data $\{(x_1,  \cdots, x_T)_m\}_{m=1}^M$.  However, realistic large-scale  stochastic simulations require the exploration of a large parameter space  and are  extremely computationally intensive.   In the following section, we describe the Interactive Neural Process (\ours{}) framework to proactively query the stochastic simulator, generate simulation data, in order to learn a fast surrogate model for rapid simulation.

\subsection{Interactive Neural Process}



 \ours{} is used to train a deep  surrogate model to mimic  the stochastic simulator. 
As shown in Figure \ref{fig:overview}, given parameters $\theta$, we query the simulator, i.e., the mechanistic  model to obtain a set of simulations $\{(x_1, \cdots, x_T)_m\}_{m=1}^M$. We train a NP based  model to learn the probabilistic map from parameters to future states. Our  NP model can be  spatiotemporal to capture complex dynamics such as the  disease dynamics of the epidemic simulator.  During inference, the model needs to generate predictions $(\hat{x}_1,\cdots, \hat{x}_T)$ at the target  parameters  $\theta$ corresponding to different scenarios.


Instead of simulating at a wide range of parameter regimes, we take a Bayesian active learning approach to proactively query the simulator and update the model incrementally.  Using NP, we can infer the latent temporal process $(z_1, \cdots, z_T)$ that encodes the  uncertainty of the current surrogate model.  Then we propose a new acquisition function, Latent Information Gain (LIG), to select the  $\theta^\star$ with the highest reward. We use  $\theta^\star$ to query the simulator, and in turn generate new simulation to further improve the model. 
Next, we describe each of the components in detail. 

\subsection{Spatiotemporal Neural Process}
Neural Process ({NP}) \citep{garnelo2018neural}  is a type of deep generative model that represents distributions over functions. It introduces a global latent variable $z$ to capture the stochasticity and learns the conditional distribution $p(x_{1:T}|\theta)$  by optimizing the evidence lower bound (ELBO):
\begin{align}
    \log p(x_{1:T}|\theta) & \geq \mathbb{E}_{q(z| x_{1:T},\theta)} \big[ \log p(x_{1:T}|z, \theta)\big] \nonumber \\
    & - \text{KL}\big( q(z|x_{1:T}, \theta) \| p(z) \big)
    \label{eqn:np}
\end{align}
Here $p(z)$ is the prior distribution for the latent variable. We use $x_{1:T}$ as a shorthand for $(x_1,\cdots, x_T)$. The prior distribution $p(z)$ is conditioned on a set of   context points $\theta^c, x_{1:T}^c$ as $p(z| x_{1:T}^c,\theta^c)$.

However, the global latent variable $z$ in NP can be limiting for non-stationary, spatiotemporal dynamics in the epidemics. We propose Spatiotemporal Neural Process (STNP) with two extensions. First, we introduce a temporal latent process $(z_1,\cdots, z_T)$ to represent the unknown dynamics.  The latent process  provides an expressive description of the internal mechanism of the stochastic simulator. Each latent variable  $z_t$ is sampled conditioning on the past history.   Second, we explicitly model the spatial dependency in  $x_t \in \mathbb{R}^D$. Rather than treating the dimensions in $x_t$ as independent features, we capture their correlations with regular grids or graphs.  For instance, the travel graph between locations can be represented as an adjacency matrix $A \in \mathbb{R}^{D\times D}$.


Given parameters $\{\theta\}$, simulation data $\{x_{1:T}\}$, and the spatial graph $A$ as inputs,  STNP models the conditional distribution $p(x_{1:T}|\theta, A)$ by optimizing the following ELBO objective: 
\begin{align}
    \log p(x_{1:T}|\theta,A) & \geq \mathbb{E}_{q(z_{1:T}| x_{1:T}, \theta,A)}  \log p(x_{1:T}|z_{1:T}, \theta,A) \nonumber \\
    & - \text{KL}\big( q(z_{1:T}|x_{1:T},\theta, A)\| p(z_{1:T}) \big)
    \label{eqn:stnp}
\end{align}
where  the  distributions $q(z_{1:T}| x_{1:T},\theta, A)$ and  $p(x_{1:T}| z_{1:T}, \theta ,A)$ are parameterized with neural networks. The prior distribution  $p(z_{1:T})$ is conditioned on a set of contextual sequences $p(z_{1:T}| x^c_{1:T}, \theta^c,A)$. Figure \ref{fig:graphical_model} visualizes the graphical models of our STNP, the original NP \citep{garnelo2018neural} model and Sequential NP \citep{singh2019sequential}. The main difference between STNP and baselines is the encoding procedure to infer the temporal latent process. Compared with STNP which directly embeds the history for $z$ inference at the current timestamp, NP ignores the history and SNP only embeds the partial history information from the previous $z$.  
\begin{figure*}[t!]
    \centering
    \includegraphics[width=0.8\linewidth]{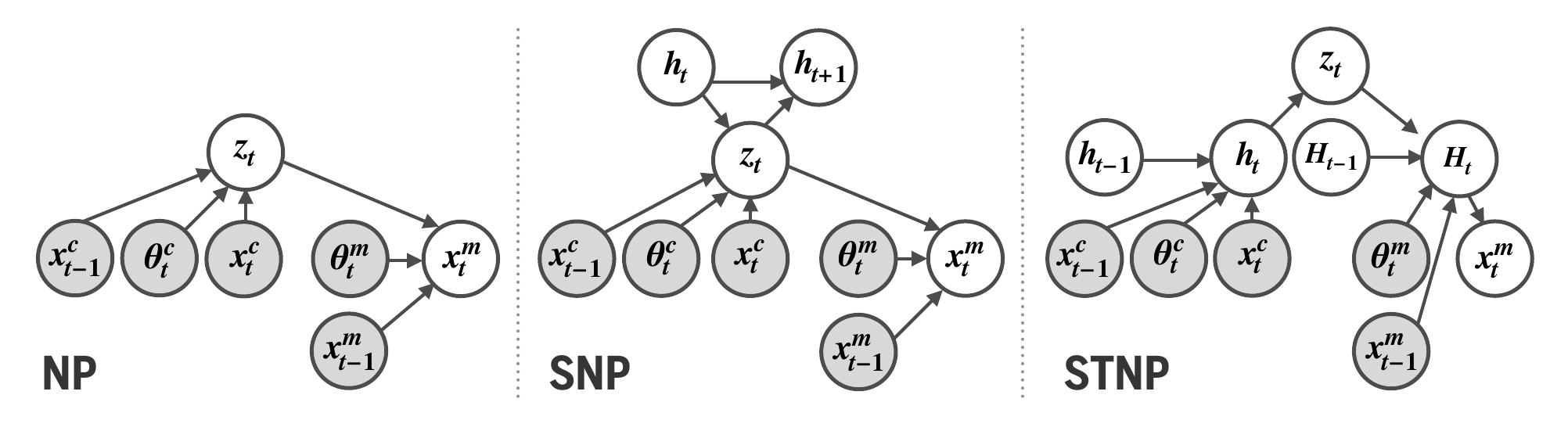}
    \caption{Graphical model comparison: Neural Process, Sequential Neural Process and our Spatiotemporal Neural Process.}
    \label{fig:graphical_model}
\end{figure*}

We implement STNP  following an encoder-decoder architecture. The encoder parametrizes the mean and standard deviation of the variational posterior $q(z_{1:T}| x_{1:T},\theta, A)$ and the decoder approximates the predictive distribution $p(x_{1:T}| z_{1:T}, \theta ,A)$. To incorporate the spatial graph information, we use a Diffusion Convolutional Gated Recurrent Unit (DCGRU) layer \citep{li2017diffusion} which integrates  graph convolution in a GRU cell. We use multi-layer GRUs to obtain  hidden states from the inputs. Using re-parametrization \citep{kingma2013auto}, we  sample $z_{t}$ from  the encoder and then  decode $x_{t}$ conditioned on $z_{t}$ in an auto-regressive fashion. To ensure fair comparisons, we adapt NP and SNP to graph-based settings and use the same architecture as STNP to generate the hidden states. Noted if the spatial dependency is regular grid-based, then the DCGRU layer is replaced to Convolutional LSTM layer \cite{lin2020preserving,wang2017predrnn,shi2015convolutional,yao2019revisiting,yao2018deep}, and there is no adjacency matrix $A$ in Equation \ref{eqn:stnp}.

\begin{algorithm}[t!]
\begin{algorithmic}
\STATE {\bfseries Input:} Initial simulation dataset $\gS_1$
\STATE Train the model $\texttt{NP}^{(1)}(\gS_1)$;
\FOR{$i =1,2, \cdots $} 

\STATE   Learn $(z_{1}, z_2,\cdots, z_T)\sim  q^{(i)}(z_{1:T}|x_{1:T}, \theta, \gS_i)$;
\STATE   Predict $(\hat{x}_{1}, \hat{x}_2,\cdots, \hat{x}_T) \sim p^{(i)}(x_{1:T}|z_{1:T}, \theta,\gS_i )$;
\STATE  Select  a batch of data: \\ 
$\{\theta^{(i+1)}\} \gets \argmax_{\theta}\Ep{p(x_{1:T}|z_{1:T}, \theta)}{r(\hat{x}_{1:T}| z_{1:T}, \theta)}$;
\STATE Simulate $\{x_{1:t}^{(i+1)}\} \gets $ Query the simulator $F(\theta^{(i+1)};\xi)$;
\STATE   Augment training set $\gS_{i+1} \gets \gS_i \cup \{\theta^{(i+1)}, x_{1:T}^{(i+1)}\}$;
\STATE   Update the model $\texttt{NP}^{(i+1)}(\gS_{i+1})$;

\ENDFOR
\end{algorithmic}
 \caption{Interactive Neural Process}
 \label{algo:inter_np}

\end{algorithm}



\subsection{Bayesian Active Learning}
Algorithm \ref{algo:inter_np} details a Bayesian active learning algorithm, based on Bayesian optimization \citep{shahriari2015taking, frazier2018tutorial}. We train an NP model to interact with the simulator and improve learning.  Let the superscript $^{(i)}$  denote the $i$-th interaction. 
 We start with an initial data set $\gS_1 = \{\theta^{(1)}, x_{1:T}^{(1)} \}$ and use it to train our NP model and learn the latent process.   During inference, given the augmented parameters $\theta$, we  use the trained NP model to predict the future states $(\hat{x}_1,\cdots, \hat{x}_T)$. We evaluate the current models' predictions with an acquisition function $r(\hat{x}_{1:T}| z_{1:T}, \theta)$ and select the set of parameters $\{\theta^{(i+1)}\}$ with the highest reward.  We query the simulator with $\{\theta^{(i+1)}\}$ to  augment the training data set  $\gS_{i+1}$ and update the NP model for the next iteration.





The choice of  the reward (acquisition) function $r$ depends on the  goal of the active learning task. For example,  to find the   model that best fits the data, the reward function can be the log-likelihood  $r= \log p(\hat{x}_{1:T}|\theta, A) $. To collect data  and reduce model uncertainty in Bayesian experimental design,  the reward function can be the mutual information. In what follows, we discuss different strategies to design the reward/acquisition function. We also propose a novel acquisition function based on information gain in the latent space tailored to our STNP model.



\subsection{Reward/Acquisition functions}
For regression tasks,  standard acquisition functions for active learning  include Maximum Mean Standard Deviation (Mean STD), Maximum Entropy,  Bayesian Active Learning by Disagreement (BALD) or expected information gain (EIG), and random  sampling \citep{gal2017deep}. We explore various acquisition functions and their approximations in the context of NP. We  also introduce a new acquisition function based on our unique NP design called {Latent Information Gain} (LIG). The details of Mean STD and Maximum Entropy are shown in the Appendix \ref{app:acquisition}.

%
\textbf{BALD/Expected Information Gain (EIG).}  
BALD \citep{houlsby2011bayesian} quantifies the mutual information between the prediction and model posterior $H(\hat{x}_{1:T}|\theta) - H(\hat{x}_{1:T}| z_{1:T},\theta) $, which is equivalent to the expected information gain (EIG). Computing the EIG for surrogate modeling is challenging since $p(\hat{x}_{1:T}| z_{1:T},\theta)$ cannot be found in closed form in general. The integrand is intractable and conventional MC methods are not applicable \citep{foster2019variational}. One way to get around this is to employ a nested MC estimator with quadratic computational cost for sampling \citep{myung2013tutorial,vincent2017darc}, which is computationally infeasible. To reduce the computational cost, we assume $p(\hat{x}_{1:T}| z_{1:T},\theta)$ follows multivariate Gaussian distribution. Each feature of $\hat{x}_{1:T}$ can be parameterized with mean and standard deviation predicted from the surrogate model, assuming output features are independent with each other. This distribution assumption can be limiting in the high-dimensional spatiotemporal domain, which makes EIG  less  informative. 

\textbf{Latent Information Gain (LIG).} 
To overcome the limitations mentioned above, we propose a novel acquisition function by computing the expected information gain in the latent space rather than  the observational space. To design this acquisition function, we prove the equivalence between the expected information gain in the observational space and the expected KL divergence in the latent processes w.r.t. a candidate parameter $\theta$, as illustrated by the following proposition.

\begin{proposition}
\label{prop:information_gain}
The expected information gain (EIG) for Neural Process is equivalent to the KL divergence between the prior and posterior in the latent process, that is 
\begin{align}
\mathrm{EIG}(\hat{x}_{1:T}, \theta) & :=\mathbb{E} [H(\hat{x}_{1:T}) - H(\hat{x}_{1:T}| z_{1:T}, \theta)] \nonumber \\
& = \mathbb{E}_{p(\hat{x}_{1:T}|\theta)} \lrb{ \mathrm{KL}\big( p( z_{1:T}| \hat{x}_{1:T}, \theta) \| p(z_{1:T})  \big) } 
\end{align}
\end{proposition}
See proof in the Appendix \ref{app:latent_info_gain}. Inspired by this fact, we propose a novel acquisition function computing the expected KL divergence in the latent processes and name it LIG. Specifically, the trained NP model produces a  variational posterior given the current dataset $\gS$ as $p(z_{1:T}|\gS)$. For every parameter  $\theta$ remained in the search space,  we can predict $\hat{x}_{1:T}$ with the decoder.  We use $\hat{x}_{1:T}$ and $\theta$ as input to the encoder to re-evaluate the posterior $p(z_{1:T}|\hat{x}_{1:T},\theta, \gS)$. LIG computes the distributional difference with respect to the latent process  $z_{1:T}$ as $\mathbb{E}_{p(\hat{x}_{1:T}|\theta)} \lrb{ \mathrm{KL}\lrp{  p(z_{1:T}| \hat{x}_{1:T},\theta,\gS)\| p(z_{1:T}|\gS) }}$, where $\mathrm{KL}(\cdot \| \cdot)$ denotes the KL-divergence between two distributions. 

In this way, conventional MC method becomes applicable, which helps reduce the quadratic computational cost to linear. At the same time, although  $z_{1:T}$ are assumed to be multivariate Gaussian and are parameterized  with mean and standard deviation, they are only in the latent space not the observational space. Moreover,  LIG is also more computationally efficient and accurate for batch active learning. Due to the context aggregation mechanism of NP, we can directly calculate LIG with respect to a batch of $\theta$ in the candidate set. This is not available for baseline acquisition functions. They all require calculating the scores one by one for all $\theta$ in the candidate set and select a batch of $\theta$ based on their scores. Such approach is both slow and inaccurate as acquiring points that are informative individually are not necessarily informative jointly \citep{kirsch2019batchbald}.




\subsection{Theoretical Analysis}
We shed light onto the intuition behind choosing adaptive sample selection over random sampling via analyzing a simplifying situation.
Assume that at a certain stage we have learned a feature map $\Psi$ which maps the input $\theta$ of the neural network to the last layer.
Then the output $X$ can be modeled as $X = \lrw{\Psi(\theta), z^*} + \epsilon$, 
where $z^*$ is the true hidden variable, $\epsilon$ is the random noise.

Our goal is to generate an estimate $\hat{z}$, and use it to make predictions $\lrw{\Psi(\theta), \hat{z}}$.
A good estimate shall achieve small error in terms of $\lrn{\hat{z}_t-z^*}_2$ with high probability.
%
In the following theorem, we prove that greedily maximizing the variance of the prediction to choose $\theta$ will lead to an error of order $\mathcal{O}({d})$ less than that of random exploration in the space of $\theta$, which is significant in high dimension.

\begin{theorem}
\label{thm:theorm}
For random feature map $\Psi(\cdot)$, greedily optimizing the KL divergence,
$\mathrm{KL}\lrp{ p(z|\hat{x},\theta) \| p(z) }$
, or equivalently the variance of the posterior predictive distribution
$\rE\lrb{ \lrp{\lrw{\Psi(\theta), \hat{z}} - \rE\lrw{\Psi(\theta), \hat{z}} }^2 }$ in search of $\theta$ will lead to an error $\lrn{\hat{z}_t-z^*}_2$ of order $\mathcal{O}\lrp{ {\sigma d}/{\sqrt{t}} }$ with high probability.
On the other hand, random sampling of $\theta$ will lead to an error of order $\mathcal{O}\lrp{ {\sigma d^{2}}/{\sqrt{t}} }$ with high probability.
\end{theorem}
See proofs in the Appendix \ref{app:thm1_proof}.
\section{Experiments}



We evaluate our proposed STNP for its surrogate modeling performance in the offline learning setting and LIG acquisition function for active learning performance.
%
We aim to verify that 
(a) LIG outperforms other acquisition functions in the NP and GP model setting for deep Bayesian active learning on non-spatiotemporal surrogate modeling, (b) STNP outperforms other existing baselines for spatiotemporal surrogate modeling in the offline learning setting, and (c) LIG outperforms other acquisition functions in the STNP model setting for deep Bayesian active learning on spatiotemporal surrogate modeling. The implementation code is available at \href{https://github.com/Rose-STL-Lab/Interactive-Neural-Process}{https://github.com/Rose-STL-Lab/Interactive-Neural-Process}.

\subsection{Experimental Setup}
We experiment with the following four stochastic simulators.

\textbf{SEIR Compartmental Model.}
To highlight the difference between NP and GP, we begin with a simple stochastic, discrete, chain-binomial SEIR compartmental model as our stochastic simulator. In this model,  susceptible individuals ($S$) become exposed ($E$) through interactions with infectious individuals ($I$) and are eventually removed ($R$), details  are deferred to the Appendix \ref{app:seir_model}.

We set the total population $N=S+E+I+R$ as $100,000$, the initial number of exposed individuals as $E_0=2,000$, and the initial number of infectious individuals as $I_0=2,000$. We assume latent individuals move to the infectious stage at a rate $\varepsilon \in [0.25,0.65]$ (step 0.05), the infectious period $\mu^{-1}$ is set to be equal to 1 day,  and we let the basic reproduction number $R_0$ (which in this case coincides with the transmissibility rate $\beta$) vary between $1.1$ and $4.0$ (step $0.1$). Here, each $(\beta,\varepsilon)$ pair corresponds to a specific scenario, which determines the parameters $\theta$. We simulate the first $100$ days of the epidemic with a total of $300$ scenarios and generate $30$ samples for each scenario.

We predict the number of individuals in the infectious compartment. The input is $(\beta,\varepsilon)$ pair and the output is the $100$ days' infection prediction. As the simulator is not spatiotemporal, we use the vanilla NP model with the global latent variable $z$.
For each epoch, we randomly select $10\%$  of the samples as context.  Implementation details are deferred to Appendix \ref{app:implementation}.

\textbf{Reaction Diffusion Model.}
The reaction-diffusion (RD) system  \citep{turing1990chemical} is a spatiotemporal model that  simulates how two chemicals might react to each other as they diffuse through a medium together. The simulation is based on initial pattern, feed rate ($\theta_0$), removal rate ($\theta_1$) and reaction between two substances. 
We use an  RD simulator  
to generate sequences from 0 to 500 timestamps, sampled every 100 timestamps, resulting into $5$ timestamps for each simulated sequence. Every timestamp is a 3D tensor $(2\times32\times32)$ with dimension $0$ corresponds to the two substances in the reaction and dimension $1,2$ are the image representation of the reaction diffusion processes. Each sequence is simulated with a unique feed rate $\theta_{0} \in [0.029,0.045]$ and kill rate $\theta_1 \in [0.055,0.062]$ combination. There are $200$  uniformly sampled  scenarios, corresponding to $(\theta_{0}, \theta_{1})$ combinations.


We implement STNP to mimic the reaction diffusion simulator with feed rate ($\theta_0$) and kill rate ($\theta_1$) as input. The initial state of the reaction is fixed. We use multiple convolutional layers  with a linear layer to encode the spatial data into latent space. We use an LSTM layer to encode the latent spatial data with $\theta_0$, $\theta_1$ to map the input-output pairs to hidden features $z_{1:5}$. With $(\theta_{0}, \theta_{1})$, and $z_{1:5}$ sampled from the posterior distribution, we use an LSTM layer and deconvolutional layers to simulate reaction diffusion sequence. For each epoch, we randomly select $20\%$ samples as context sequence.

\textbf{Heat Model.}
The model is to predict the spatial solution fields of the Heat equation \cite{olsen2011numerical}. The ground-truth data is generated from the standard numerical solver used in \cite{li2020deep}. The experiment setting also follows \cite{li2020deep}. The examples are generated by solvers running with $32 \times 32$ meshes. The corresponding output dimension is $1024$. The input consists of three parameters that control the thermal conductivity and the flux rate.



\textbf{Local Epidemic and Mobility model.}
The Local Epidemic and Mobility model (LEAM-US) is a stochastic, spatial, age-structured epidemic model based on a metapopulation approach which divides the US in more than 3,100 subpopulations, each one corresponding to a each US county or statistically equivalent entity.  Population size and county-specific age distributions reflect Census' annual resident population estimates for year 2019. We consider individuals divided into $10$ age groups. Contact mixing patterns are age-dependent and state specific and modeled considering contact matrices that describe the interaction of individuals in different social settings \citep{mistry2020inferring}.
LEAM-US integrates a human mobility layer, represented as a network, using both short-range (i.e., commuting) and long-range (i.e., flights) mobility data, see more details in Appendix \ref{app:leam_model}.

We separate data in California monthly to predict the 28 days' sequence from the 2nd to the 29th day of each month from March to December. Each $\theta$ includes the county-level parameters of LEAM-US and state level incidence and prevalence compartments. The total number of dimension in $\theta$ is $16,912$, see details in Appendix \ref{app:leam_model}.   Overall, there are $315$ scenarios in the search space, corresponding to $315$ different $\theta$ with total $16,254$ samples. We split $78\%$ of the data as the candidate set, and $11\%$ for validation and test.
For active learning, we use the candidate set as the search space.

We instantiate an STNP model to mimic an epidemic simulator that has $\theta$ at both county and state level and $x_t$ at the state level. We use county-level parameter $\theta$ together with a county-to-county mobility graph $A$ in California as input.  We use the DCGRU layer \citep{li2017diffusion} to encode the mobility graph in a GRU. We use a linear layer to map the county-level output to hidden features  at the state level. For both the state-level encoder and  decoder, we use multi-layer GRUs. For each epoch, we randomly select $20\%$ samples as context sequence.

\subsection{Offline Learning Performance}

We compared our proposed STNP with vanilla NP \cite{garnelo2018neural}, SNP \cite{singh2019sequential}, Masked Autoregressive Flow (MAF) \cite{papamakarios2017masked}, the RNN baseline with variational dropout (RNN) \cite{zhu2017deep}, and VAE-based deep surrogate model for multi-fidelity active learning (DMFAL) \cite{li2020deep}. The implementation details can be seen in Appendix \ref{app:implementation_offline}. The key innovation of STNP is the introduced temporal latent process. To ensure fair comparison, we modified baselines for the RD and Heat model by adding convolutional layers for data encoding and deconvolutional layers for sequence generation. For the LEAM-US model, we modified NP by adding the convolutional layers with diffusion convolution  \cite{li2018diffusion} to embed the graphs. Similarly, we modified SNP by replacing the convolutional layers with diffusion convolution. The rest baselines do not support LEAM-US surrogate modeling. Table \ref{tb:offline} shows the testing MAE of different NP models trained in an offline fashion. Our STNP significantly improves the performance and can accurately learn the simulator dynamics for all three experiments. 

\begin{figure*}[t]
    \centering
    \includegraphics[width=\linewidth]{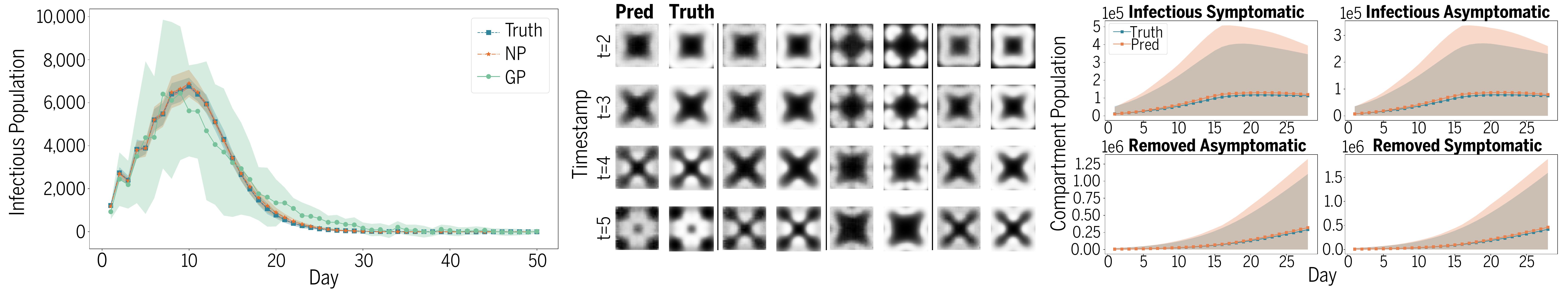}
    \caption{Prediction visualizations, Left: Accuracy and uncertainty quantification comparison between Neural Process (NP) and Gaussian process (GP) in SEIR simulator. Middle: STNP predictions for spatiotemporal patterns of substances in Reaction-Diffusion simulator. Right: STNP predictions for the number of individuals in Infectious and Removed compartments in LEAM-US simulator.}
    \label{fig:pred_visualization}
\end{figure*}

\begin{table*}
\centering
\caption{Surrogate model performance comparison using MAE in Reaction-Diffusion, Heat, and LEAM simulator (population divided by 1000).}
\begin{tabular}{c|c|c|c|c|c|c}
\toprule
Model &	NP	& SNP &	MAF & RNN & DMFAL &	STNP \\
\midrule
RD & 3.37 ± 0.18 & 3.11 ± 0.07 & 7.22 ± 0.66 & 3.44 ± 0.07 & 4.1 ± 0.02 & \textbf{2.84 ± 0.17}\\
 \midrule
 HEAT & 	1.05e-2 ± 1.1e-3 & 9.62e-3 ± 1e-3 & 2.1e-2 ± 4.8e-3 & 3.2e-2 ± 1.7e-3 & 1.35e-2 ± 3e-4 & \textbf{7.36e-3 ± 6.7e-4} \\
 \midrule 
 LEAM-US & 24.2 ± 5.9 & 21.8 ± 0.8 & -- & -- & -- & \textbf{6.3± 0.8} \\
\bottomrule

\end{tabular}
\label{tb:offline}
\end{table*}

\begin{figure*}[t]
    \centering
    \includegraphics[width=\linewidth]{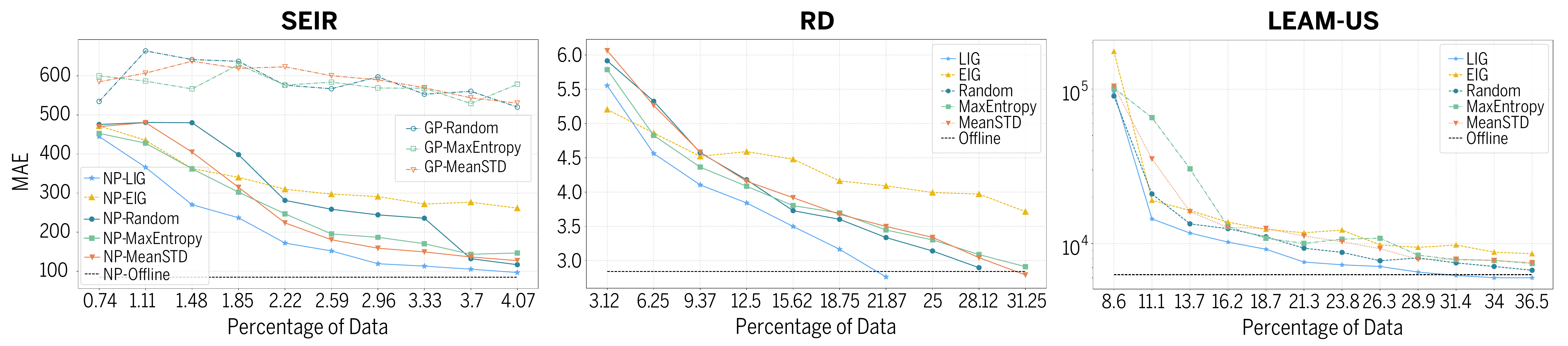}
    \caption{MAE loss versus the percentage of samples for Bayesian active learning. The Black
dash line shows the offline learning performance with the entire data set available for training. Left: GP and NP for SEIR. Middle: STNP for RD. Right: STNP for LEAM-US.}
    \label{fig:active_performance}
\end{figure*}

Figure \ref{fig:pred_visualization} left compares the NP and GP performance on one scenario in the held-out test set. It shows the ground truth and the predicted number of infectious population for the first $50$ days. We also include the confidence intervals (CI) with $5$ standard deviations for ground truth and NP predictions and $1$ standard deviation for GP predictions. We observe that NP fits the simulation dynamics better than GP for mean prediction. Moreover, NP has closer CIs to the truth, reflecting  the simulator's intrinsic uncertainty. GP shows larger CIs which represent the model's own uncertainty. Note that NP is much more flexible than GP and can scale easily to high-dimensional data. Figure \ref{fig:pred_visualization} middle indicates STNP can accurately predict various patterns corresponding to different $(\theta_0$, $\theta_1)$. 
This confirms that our STNP is able to capture the high-dimensional spatiotemporal dependencies in RD simulations. Figure \ref{fig:pred_visualization} right visualize the STNP predictions in four key compartments of a typical scenario with $R_{0}=3.1$ from March 2nd to March 29th. The confidence interval is plotted with $2$ standard deviations. We can see that both the mean and confidence interval of STNP predictions match the truth well. These two results demonstrate the promise that the  generative STNP model can serve as a deep surrogate model for RD and LEAM-US simulator.

\subsection{Active Learning}
\textbf{Implementation Details.}
We compare $6$ different acquisition functions with NP for SEIR model and STNP for RD, Heat, and LEAM-US model. For SEIR, the initial training dataset has 2 scenarios and we continue adding $1$ scenario per iteration to the training set until the test loss converges to the offline modeling performance. We also include GP with $3$ different acquisition functions and Sequential Neural Likelihood (SNL) (see Table \ref{tb:SEIR_GP}). For the RD and Heat model, all acquisition functions start with the same $5$ scenarios randomly picked from the training dataset. Then we continue adding $5$ scenarios per iteration to the training set until the test loss converges. Similarly, the LEAM-US model begins with $27$ training data and we continue adding $8$ scenarios per iteration to the training set until the validation loss converges. We measure the average performance over three random runs and report the MAE for the test set.

\textbf{Active Learning Performance.}
Figure \ref{fig:active_performance} shows the testing MAE versus the percentage of samples included for training. The percentage of data is linearly proportional to the overall running time. This figure shows our proposed LIG always has the best MAE performance until the convergence for SEIR, RD, and LEAM-US tasks. As shown in figure \ref{fig:active_performance} left, we compare different acquisition functions on both NP and GP for SEIR model. It shows none of the GP methods converge after selecting $4.07\%$ of the data for training while NP methods converge much faster. Our proposed acquisition function LIG is the most sample efficient in acquisition functions used for NP. It takes only $4.07\%$ of the data to converge and reach the NP offline performance, which uses the entire training set for training. Moreover, there is an enormous gap between LIG and EIG with respect to the active learning performance. This validates our theory that the uncertainty of the deep surrogate model is better measured on the latent space instead of the predictions. Similarly in figure \ref{fig:active_performance} middle and right, we compared LIG with other acquisition functions on STNP for RD and LEAM-US model. It shows LIG converges to the offline performance using only $21.87\%$ of data for RD experiment and $31.4\%$ of data for LEAM-US experiment. Therefore, it is consistent among all three experiments that our proposed LIG always has the best MAE performance until convergence. Notice that for figure \ref{fig:active_performance} right, it shows the log scale MAE versus the percentage of samples included for training. Detailed performance comparison including mean and standard deviation for all four tasks including Heat can be seen in Appendix \ref{app:active_learning}. Our proposed LIG also has the best MAE performance for the Heat task.


\begin{figure*}[t!]
    \centering
    \includegraphics[width=\linewidth]{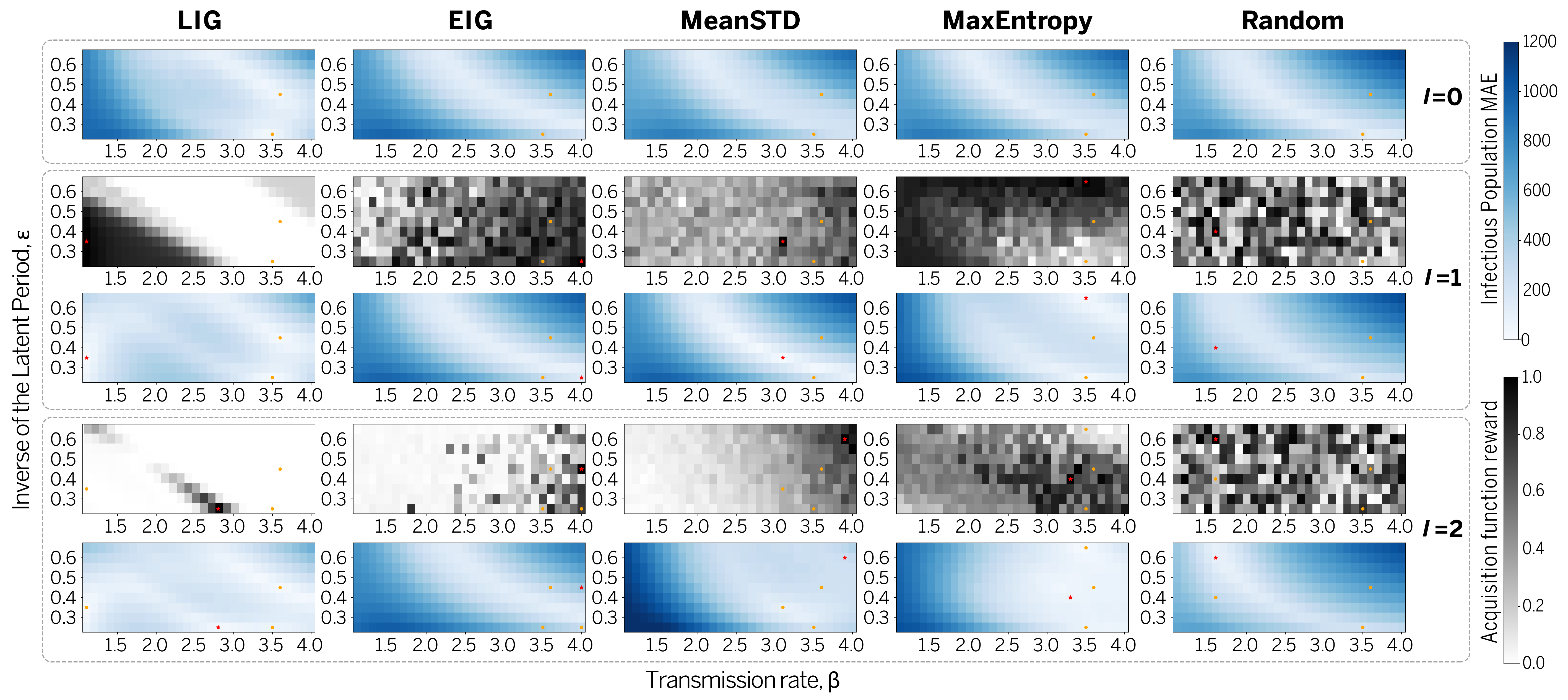}
    \caption{Acquisition function behavior visualization in SEIR model. For each iteration, top row is the current MAE mesh in infectious population for all $(\beta,\varepsilon)$ candidates.
    Bottom row is the acquisition function score. Yellow dots are existing parameters. Red stars are the newly selected parameters. }
    \label{fig:seir_acq_visual}
\end{figure*}

\textbf{Exploration Exploitation Trade-off.}
To understand the large performance gap for LIG vs. baselines, we visualize the values of test MAE and the acquisition function score for each Bayesian active learning iteration for SEIR model, shown in Figure \ref{fig:seir_acq_visual}.  For EIG, Mean STD, and Maximum Entropy, they all tend to exploit the region with large transmission rate for the first $2$ iterations. Including these scenarios makes the training set unbalanced. The MAE in the region with small transmission rate become worse after $2$ iterations. Meanwhile, Random  is doing pure exploration. The improvement of MAE performance is not apparent after $2$ iterations. Our proposed LIG 
is able to reach a balance by exploiting the uncertainty in the latent process and encouraging exploration. Hence, with a small number of iterations ($I=2$), it has already selected ``informative scenarios'' in the search space.

\section{Conclusion}
We propose a unified framework Interactive Neural Processes (\ours{}) for deep Bayesian active learning, that can seamlessly interact with existing stochastic simulators and accelerate simulation. Specifically,   we design STNP to approximate the underlying simulation dynamics. It infers the latent process which describes the intrinsic uncertainty of the simulator. We exploit this uncertainty and propose LIG as a powerful acquisition function in deep Bayesian active learning.  We perform a theoretical analysis and demonstrate that our approach reduces sample complexity compared with random sampling in high dimension. We also did extensive empirical evaluations on several complex real-world spatiotemporal simulators to demonstrate the superior performance of our proposed STNP and LIG. For the future work, we plan to leverage Bayesian optimization techniques to directly optimize for the target parameters  with auto-differentiation. 




\section*{Acknowledgments}
This work was supported in part by U.S. Department Of Energy, Office of Science, Facebook Data Science Research Awards, U. S. Army Research Office under Grant W911NF-20-1-0334, and NSF Grants \#2134274 and \#2146343, as well as NSF-SCALE MoDL (2134209) and NSF-CCF-2112665 (TILOS). M.C. and A.V. acknowledge support from grant HHS/CDC 5U01IP0001137.

\bibliographystyle{ACM-Reference-Format}
\bibliography{ref}

\clearpage

\newpage
\onecolumn
\appendix
\section{Theoretical Analysis}
\subsection{Latent Information Gain}
\label{app:latent_info_gain}
\setcounter{proposition}{0}
\begin{proposition}
The expected information gain (EIG) for Neural Process is equivalent to the KL divergence between the prior and posterior in the latent process, that is 
\begin{align}
\mathrm{EIG}(\hat{x}, \theta) :=\mathbb{E} [H(\hat{x}) - H(\hat{x}| z, \theta)] =
\mathbb{E}_{p(\hat{x}|\theta)} \lrb{ \mathrm{KL}\big( p( z| \hat{x}, \theta) \| p(z)  \big) }
\end{align}
\end{proposition}
\begin{proof}[Proof of Proposition~1]
The information gained in the latent process $z$, by selecting the parameter $\theta$ and generate $\hat{x}$ is the reduction in entropy from the prior to the posterior
$\mathrm{IG} (\theta) = H(\hat{x}) -  H(\hat{x}|z, \theta)$. 
Take the expectation of $\mathrm{IG}( \hat{x},\theta)$ under the marginal distribution, we obtain from the conditional independence of $z$ and $\theta$ that



\begin{align*}
&\mathbb{E}_{p(\hat{x}|\theta)} \lrb{ \mathrm{KL}\big( p( z| \hat{x}, \theta) \| p(z)  \big) } \\
&= \mathbb{E}_{p(\hat{x},z|\theta)} \bigg[ \log \frac{p(z|\hat{x}, \theta )}{ p( z)} \bigg]\\
&= \mathbb{E}_{p(\hat{x},z|\theta)} \bigg[ \log \frac{p(z|\hat{x}, \theta )}{ p( z|\theta)} \bigg]\\
&= \mathbb{E}_{p(\hat{x},z|\theta)} \big[ \log {p(z, \hat{x}, \theta )} - \log p(\hat{x}, \theta) - { \log p( z,\theta)} + \log p(\theta) \big] \\
&= \mathbb{E}_{p(\hat{x},z|\theta)} \big[ \log {p(\hat{x} |z, \theta )} - \log p(\hat{x} | \theta) \big] \\
&= \mathbb{E}_{p(z)} \bigg[ \mathbb{E}_{p(\hat{x}|z, \theta)} [\log {p(\hat{x} |z, \theta )} ] - \mathbb{E}_{p(\hat{x} | \theta)} [\log p(\hat{x} | \theta)] \bigg] \\
&= \mathbb{E}_{p(z)} [H(\hat{x} | \theta)-H(\hat{x} |z, \theta )]\\
&= \mathrm{EIG}(\hat{x}, \theta).
\end{align*}

\end{proof}
\subsection{Sample Efficiency of Active Learning}
\label{app:thm1_proof}
From the main text we know that in each round, the output random variable \begin{align}
X = \lrw{\Psi(\theta),z^*} + \epsilon.
\label{eq:model}
\end{align}
We further assume that the random noise $\epsilon$ is mean zero and $\sigma$-subGaussian.

Using this information, we treat $z$ as an unknown parameter and define a likelihood function so that $p(X|z;\theta)$ has good coverage over the observations:
\[
p(X_k|z;\theta_k) \propto\exp\lrp{ -\frac{1}{2\sigma^2} \lrp{X_k-\lrw{\Psi(\theta_k), z}}^2 }.
\]

Let the prior distribution over $z$ be $p(z|\theta_k) = p(z) \propto \exp\lrp{-\frac{m}{2\sigma^2} \lrn{z}^2}$.
Here we use $k$ instead of $(i)$ in the Algorithm \ref{algo:inter_np} to represent the number of iterations. We can form a posterior over $z$ in the $k$-th round:
\[
p(z|X_1,\theta_1,\dots,X_k,\theta_k) \propto \exp\lrp{ -\frac{m}{2\sigma^2} \lrn{z}^2 - \frac{1}{2\sigma^2} \sum_{s=1}^k \lrp{X_s - \lrw{\Psi(\theta_s), z}}^2 }.
\]
Focusing on the random variable $z\sim p(\cdot|X_1,\theta_1,\dots,X_k,\theta_k)$, the estimate of the hidden variable, we can express it at $k$-th round as:
\begin{align}
z_k = \hat{z}_k + \sigma V_k^{-1} \eta_k,
\label{eq:var_change}
\end{align}
where $\hat{z}_k = V_k^{-1} \sum_{s=1}^k X_s \Psi(\theta_s)$,
$V_k = m\mI + \sum_{s=1}^k \Psi(\theta_s) \Psi(\theta_s)^\rT$,
and $\eta_k$ is a standard normal random variable.

We can either choose action $\theta$ randomly or greedily.
A random choice of $\theta$ corresponds to taking
\begin{align}
\theta_k \sim \mathcal{N}\lrp{0,\mI},
\label{eq:random}
\end{align}
A greedy procedure is to choose action $\theta_k$ in the $k$-th round 
to optimize $\mathrm{KL}\lrp{p(z|\hat{x},\theta)\|p(z)} = \mathbb{E}_{p(z|\hat{x},\theta)} \lrp{\log \frac{p(z|\hat{x},\theta)}{p(z)}}$, where we denote the estimated output variable $\hat{x}$ given $\theta$ and $z$ as $\hat{x}=\lrw{\Psi(\theta),z}$.
This optimization procedure is equivalent to maximizing the variance of the prediction:
\begin{align}
\theta_k = \arg\max_{\theta\in\R^d} \rE_{z\sim p(\cdot|X_1,\theta_1,\dots,X_{k-1},\theta_{k-1})}\lrb{ \lrp{\lrw{\Psi(\theta), {z}} - \rE_{z\sim p(\cdot|X_1,\theta_1,\dots,X_{k-1},\theta_{k-1})}\lrw{\Psi(\theta), {z}} }^2 }.
\label{eq:greedy}
\end{align}
For both approaches, we assume that the features $\Psi(\theta)$ are normalized.

We compare the statistical risk of this approach with the random sampling approach.

Assume that the features are normalized, so that for all $\theta\in\R^d$, $\Psi(\theta)\in\mathbb{S}^{d-1}$.
Define a matrix $\mA_k\in\R^{d\times k}$ containing all the column vectors $\lrbb{\Psi(\theta_1),\dots,\Psi(\theta_k)}$.
We can then express the estimation error in the following lemma.
\begin{lemma}
The estimation error $\lrn{\hat{z}_k-z^*}_2$ can be bounded as follow.
\begin{align*}
\lrn{\hat{z}_k-z^*}_2
&\leq m \lrp{ m + \sigma_{\min}\lrp{\mA_k \mA_k^\rT} }^{-1} \cdot \lrn{z^*}_2 \\
&+\min\lrbb{1/\lrp{2\sqrt{m}},1/\lrp{ \sqrt{\sigma_{\min}\lrp{\mA_k \mA_k^\rT}} + \frac{m}{\sqrt{\sigma_{\min}\lrp{\mA_k \mA_k^\rT}}} }} \cdot \sigma \sqrt{d}.
\end{align*}
\label{lem:error}
\end{lemma}

We now analyze random sampling of $\theta$ versus greedy search for $\theta$.

If the feature map $\Psi(\cdot) = \mathrm{id}$, then from random matrix theory, we know that for $\theta$ randomly sampled from a normal distribution and normalized to $\lrn{\theta}=1$, $\sigma_{\min}\lrp{\frac{1}{k}\mA_k \mA_k^\rT} $ will converge to $ \lrp{ \sqrt{1/k}-\sqrt{1/d} }^2$ for large $k$, which is of order $\Omega(1/d)$.
This will lead to an appealing risk bound for $\lrn{\hat{z}_k-z^*}_2$ on the order of $\mathcal{O}\lrp{ d/\sqrt{k} }$.

However, in high dimension, this feature map is often far from identity.
In the proof of Theorem~1 below, we demonstrate that even when $\Psi(\cdot)$ is simply a linear random feature map, with i.i.d. normal entries, random exploration in $\theta$ can lead to deteriorated error bound.
This setting is motivated by the analyses in wide neural networks, where the features learned from gradient descent are close to those generated from random initialization~\cite{Lee2019NTK,Andrea2019}.

\setcounter{theorem}{0}
\begin{theorem}[Formal statement]
Assume that the noise $\epsilon$ in~\eqref{eq:model} is $\sigma$-subGaussian.

For a normalized linear random feature map $\Psi(\cdot)$, greedily optimizing the KL divergence, $\mathrm{KL}\lrp{p(z|\hat{x},\theta)\|p(z)}$ (or equivalently the variance of the posterior predictive distribution defined in equation~\eqref{eq:greedy}) in search of $\theta$ will lead to an error $\lrn{\hat{z}_k-z^*}_2 = \mathcal{O}\lrp{ {\sigma d}/{\sqrt{k}} }$ with high probability.

On the other hand, random sampling of $\theta$ following~\eqref{eq:random} will lead to $\lrn{\hat{z}_k-z^*}_2 = \mathcal{O}\lrp{ {\sigma d^{2}}/{\sqrt{k}} }$ with high probability.
\end{theorem}

\begin{proof}[Proof of Theorem~1]
For a linear random feature map, we can express $\Psi(\theta) = \Psi \theta$, where entries in $\Psi\in\R^{d\times d}$ are i.i.d. normal.
The entries of $\Psi \theta$ are then normalized.

\begin{itemize}
    \item 
    For random exploration of $\theta$, the matrix containing the feature vectors becomes $\mA_k = \Psi \Theta_k$, where matrix $\Theta_k\in\R^{d\times k}$ collects all the $k$ column vectors of $\{\theta_1,\dots,\theta_k\}$.
Then $\mA_k\mA_k^\rT = \Psi \Theta_k \Theta_k^\rT \Psi^\rT$.
From random matrix theory, we know that the condition number of $\Psi$ is equal to $d$ with high probability~\cite{random_matrix_cond}.
Hence for normalized $\Psi$ and $\theta$, $\sigma_{\min} \lrp{ \Psi \Theta_k \Theta_k^\rT \Psi^\rT } \geq \sigma_{\min}^2\lrp{\Psi} \sigma_{\min} \lrp{ \Theta_k \Theta_k^\rT } = \frac{1}{d^2} \sigma_{\min}\lrp{\Theta_k\Theta_k^\rT}$.
The inequality holds because the smallest singular value is the inverse of the norm of the inverse matrix.

We then use the fact from random matrix theory that for normalized random $\theta$, the asymptotic distribution of the eigenvalues of $\frac{1}{k} \Theta_k \Theta_k^\rT$ follow the (scaled) Marchenko–Pastur distribution, which is supported on $\lambda\in\lrb{ \lrp{ \sqrt{1/k}-\sqrt{1/d} }^2, \lrp{ \sqrt{1/k}+\sqrt{1/d} }^2 }$, where the $1/d$ scaling comes from the fact that $\theta$ is normalized~\cite{random_matrix_eig}.
Hence for large $k$, $\sigma_{\min} \lrp{\Theta_k\Theta_k^\rT} \geq \lrp{1-\sqrt{k/d}}^2$ with high probability.
This combined with the previous paragraph yields that for the random feature model, 
$$\sigma_{\min} \lrp{\mA_k\mA_k^\rT} = \Omega\lrp{ \frac{1}{d^2} \lrp{1-\sqrt{k/d}}^2 }$$ 
with high probability.
Plugging this result into Lemma~\ref{lem:error}, we obtain that the error $\lrn{\hat{z}_k-z^*}_2$ for random exploration in the space of $\theta$ is of order $\mathcal{O}\lrp{d^2/\sqrt{k}}$.

\item
We then analyze the error associated with greedy maximization of the posterior predictive variance.
We first note that the variance of the posterior predictive distribution in equation~\eqref{eq:greedy} can be expressed as follows using equation~\eqref{eq:var_change}:
\begin{align}
\rE\lrb{ \lrp{\lrw{\Psi(\theta), z} - \rE\lrw{\Psi(\theta), z} }^2 }
= \sigma^2 \rE\lrb{ \lrp{\lrw{\Psi(\theta),V_{k-1}^{-1}\eta_k}}^2 }
= \sigma^2 \Psi(\theta)^\rT V_{k-1}^{-2} \Psi(\theta),
\label{eq:quadratic}
\end{align}
where the expectations are with respect to $z\sim p(\cdot|X_1,\theta_1,\dots,X_{k-1},\theta_{k-1})$.

We perform a singular value decomposition $\mA_k = U_k \Lambda_k W_k$.
Then $\sum_{s=1}^k \Psi(\theta_s) \Psi(\theta_s)^\rT = \mA_k \mA_k^\rT = U_k \Lambda_k \Lambda_k^\rT U_k^\rT$, and that $V_{k-1}^{-2} = \lrp{m\mI+A_{k-1} A_{k-1}^\rT}^{-2} = U_{k-1} \lrp{m\mI+\Lambda_{k-1} \Lambda_{k-1}^\rT}^{-2} U_{k-1}^\rT$.
Via this formulation, we see that maximizing ${\Psi(\theta)^\rT V_{k-1}^{-2} \Psi(\theta)}$ in equation~\eqref{eq:quadratic} to choose $\theta_k$ is equivalent to choosing $\Psi(\theta_k) = \lrp{U_{k-1}}^\rT_{\lrp{\cdot,l}}$, where $l=\arg\min_{i\in\{1,\dots,d\}} \lrp{\Lambda_{k-1}\Lambda_{k-1}^\rT}_{\lrp{i,i}}$.
In words, when we use greedy method and maximize the variance of the prediction, it corresponds to taking $\Psi(\theta_k)$ in the direction of the smallest eigenvector of $V_{k-1}$.

Since every $\Psi(\theta)$ is normalized and we initialize uniformly: $V_0 = m\mI$, the process is equivalent to scanning the orthogonal spaces of normalized vectors in $\R^d$ for $\lfloor {k}/{d} \rfloor$ times.
For large $k$, entries in $\Lambda_k \Lambda_k^\rT$ are approximately uniform and are all larger than or equal to $\lfloor {k}/{d} \rfloor$.
Then $\sigma_{\min}\lrp{\mA_k \mA_k^\rT} = \Omega(k/d)$.
Plugging into the bound of Lemma~\ref{lem:error}, we obtain that
\[
\lrn{\hat{z}_k-z^*}_2
= \mathcal{O}\lrp{ \frac{\sigma d}{\sqrt{k}} }.
\]
\end{itemize}
\end{proof}

\begin{proof}[Proof of Lemma~\ref{lem:error}]
We first express the estimate $\hat{z}_k$ as follows.
\begin{align*}
\hat{z}_k
= V_k^{-1} \sum_{s=1}^k X_s \Psi(\theta_s)
= V_k^{-1} \sum_{s=1}^k \Psi(\theta_s) \Psi(\theta_s)^\rT z^* + V_k^{-1} \sum_{s=1}^k \epsilon_s \Psi(\theta_s).
\end{align*}
Then
\begin{align*}
\lrn{\hat{z}_k-z^*}_2
&= \lrn{ \lrp{ V_k^{-1} \sum_{s=1}^k \Psi(\theta_s) \Psi(\theta_s)^\rT - \mI} z^* + V_k^{-1} \sum_{s=1}^k \epsilon_s \Psi(\theta_s) }_2 \\
&\leq \underbrace{\lrn{\lrp{ V_k^{-1} \sum_{s=1}^k \Psi(\theta_s) \Psi(\theta_s)^\rT - \mI} z^*}_2}_{T_1} + \underbrace{\lrn{V_k^{-1} \sum_{s=1}^k \epsilon_s \Psi(\theta_s)}_2}_{T_2}.
\end{align*}
Define a matrix $\mA_k\in\R^{d\times k}$ containing all the column vectors $\lrbb{\Psi(\theta_1),\dots,\Psi(\theta_k)}$ and perform a singular value decomposition $\mA_k = U_k \Lambda_k W_k$.
Then $\sum_{s=1}^k \Psi(\theta_s) \Psi(\theta_s)^\rT = \mA_k \mA_k^\rT = U_k \Lambda_k \Lambda_k^\rT U_k^\rT$, and $V_k = m\mI + \mA_k \mA_k^\rT$,
We further define vector $\E_k\in\R^{s}$ where $(\E_k)_{s} = \epsilon_s$.
We use this definition to simplify the two terms further.

For term $T_1$, 
\begin{align*}
\lrn{\lrp{ V_k^{-1} \sum_{s=1}^k \Psi(\theta_s) \Psi(\theta_s)^\rT - \mI} z^*}_2
&= m \lrn{ V_k^{-1} z^* }_2 \\
&\leq m \lrn{ V_k^{-1} }_2 \cdot \lrn{z^*}_2 \\
&= m \lrp{ m + \sigma_{\min}\lrp{\mA_k \mA_k^\rT} }^{-1} \cdot \lrn{z^*}_2.
\end{align*}

For term $T_2$, we define a diagonal matrix $\Bar{\Lambda}_k\in\R^{k\times k}$ which satisfies $\lrp{\Bar{\Lambda}_k}_{i,i}=1$ if $i\leq d$ and $\lrp{\Bar{\Lambda}_k}_{i,i}=0$ if $i> d$, when $k>d$.
The following bound on $T_2$ can be achieved.
\begin{align*}
\lrn{V_k^{-1} \sum_{s=1}^k \epsilon_s \Psi(\theta_s)}_2
&= \lrn{V_k^{-1} \mA_k \E_k}_2 \\
&= \lrn{ U_k \lrp{ \Lambda_k \Lambda_k^\rT + m \mI }^{-1} U_k^\rT U_k \Lambda_k \Bar{\Lambda}_k W_k \E_k}_2 \\
&\leq \lrn{U_k \lrp{ \Lambda_k \Lambda_k^\rT + m \mI }^{-1} \Lambda_k}_2 \cdot \lrn{\Bar{\Lambda}_k W_k \E_k}_2 \\
&= \lrn{\lrp{ \Lambda_k \Lambda_k^\rT + m \mI }^{-1} \Lambda_k}_2 \cdot \lrn{\Bar{\Lambda}_k W_k \E_k}_2 \\
&\leq \min\lrbb{1/\lrp{2\sqrt{m}},1/\lrp{ \sqrt{\sigma_{\min}\lrp{\mA_k \mA_k^\rT}} + \frac{m}{\sqrt{\sigma_{\min}\lrp{\mA_k \mA_k^\rT}}} }} \cdot \lrn{\Bar{\Lambda}_k W_k \E_k}_2.
\end{align*}
Assuming that noise $\epsilon_s$ is $\sigma$-subGaussian, then so is $W_k \E_k$ since $W_k$ is a unitary matrix.
Multiplied by the diagonal matrix $\Bar{\Lambda}_k$ which has zero, $\lrn{\Bar{\Lambda}_k W_k \E_k}_2 \leq \sigma \sqrt{d}$.
Therefore, 
\[
\lrn{V_k^{-1} \sum_{s=1}^k \epsilon_s \Psi(\theta_s)}_2 
\leq \min\lrbb{1/\lrp{2\sqrt{m}},1/\lrp{ \sqrt{\sigma_{\min}\lrp{\mA_k \mA_k^\rT}} + \frac{m}{\sqrt{\sigma_{\min}\lrp{\mA_k \mA_k^\rT}}} }} \cdot \sigma \sqrt{d}.
\]
\end{proof}

\section{Experiment Details}
\subsection{SEIR Model}
\label{app:seir_model}
Our SEIR simulator is a simple stochastic, discrete, chain-binomial compartmental model. In this model, susceptible individuals ($S$) become exposed ($E$) through interactions with infectious individuals ($I$). Exposed individuals which are infected but not yet infectious transition to infectious compartment at a rate $\varepsilon$ that is inversely proportional to the latent period of the disease. Lastly, infectious individuals transition to the removed compartment at a rate $\mu$ which is inversely proportional to the infectious period. Removed individuals ($R$) are assumed to be no longer infectious and they are to be considered either recovered or dead. All transitions are simulated by randomly drawn from a binomial distribution.

\subsection{LEAM-US Model}
\label{app:leam_model}
LEAM-US integrates a human mobility layer, represented as a network, using both short-range (i.e., commuting) and long-range (i.e., flights) mobility data. 
Commuting flows between counties are obtained from the 2011-2015 5-Year ACS Commuting Flows survey and properly adjusted to account for differences in population totals since the creation of the dataset. Instead, long-range air traveling flows are quantified using origin-destination daily passenger flows between airport pairs as reported by the Official Aviation Guide (OAG) and IATA databases (updated in 2021) \citep{OAG,IATA}. In addition, flight probabilities are age and country specific.

The model is initialized using a multi-scale modeling approach that utilizes GLEAM, the Global and Epidemic Mobility model \citep{balcan2009multiscale,balcan2010modeling,tizzoni2012real,zhang2017spread,chinazzi2020effect,davis2020estimating}, to simulate a set of 500 different initial conditions for LEAM-US starting on February 16th, 2020. The disease dynamics are modeled using a classic SEIR-like model and initial conditions are determined using the Global and Epidemic Mobility model \citep{balcan2009multiscale,balcan2010modeling,tizzoni2012real,zhang2017spread} calibrated to realistically represent the evolution of the COVID-19 pandemic \citep{chinazzi2020effect,davis2020estimating}.
Lastly, travel restrictions, mobility reductions, and government interventions are explicitly modeled to mimic the real timeline of interventions of the events that occurred during the COVID-19 pandemic. 

\subsection{Spatiotemporal NP Model}

\begin{figure}[h]
  \centering
  \includegraphics[width=0.7\linewidth,trim={0 0 0 50}]{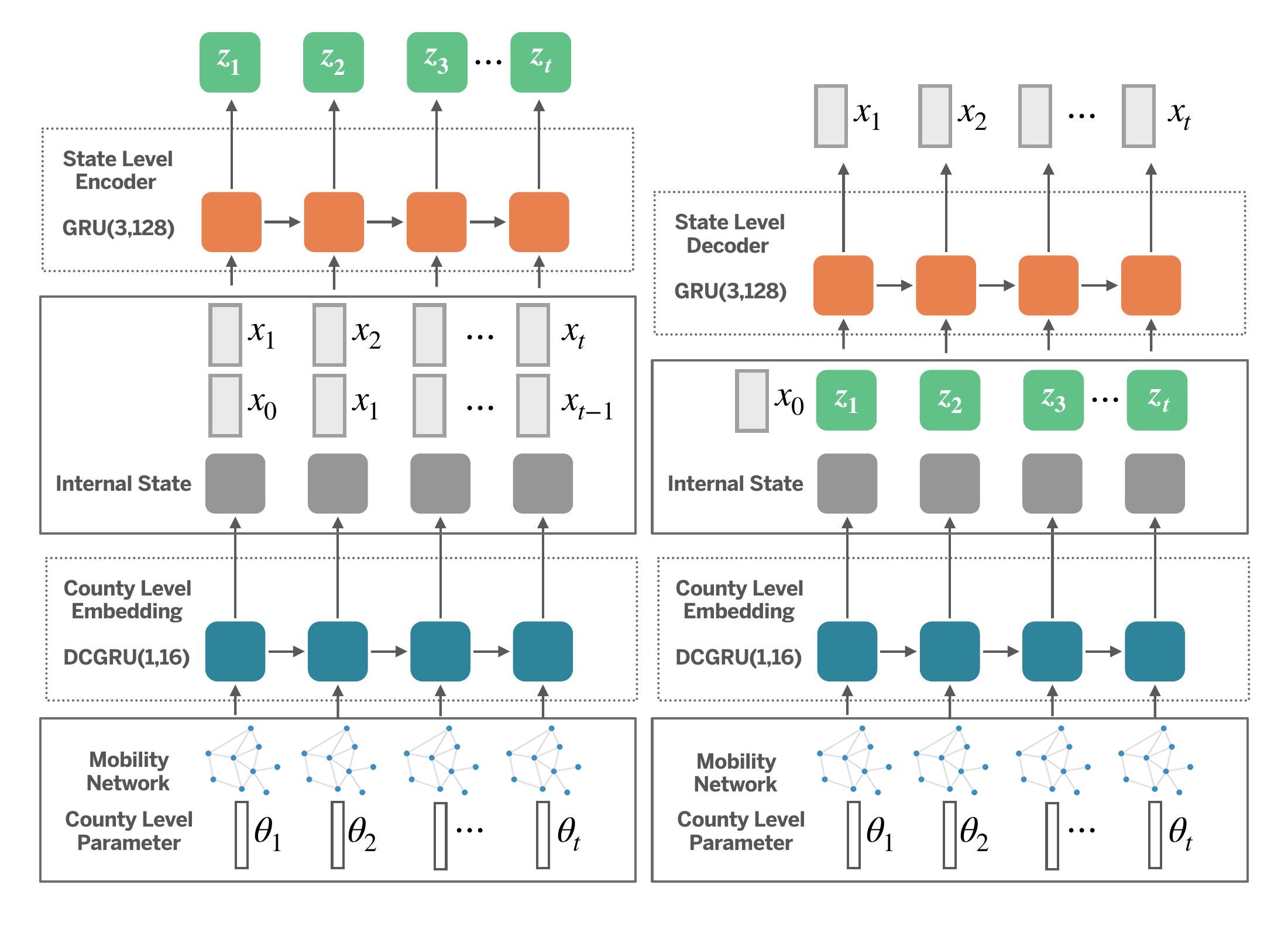}
  \caption{Visualization of the STNP model architecture. For both the encoder and the decoder, we use a  diffusion convolutional GRU (DCGRU) \cite{li2018diffusion} to capture spatiotemporal dependency. }
  \label{fig:snp}
\end{figure}

As shown in Figure \ref{fig:snp}, our model has $\theta$ at both county and state level and $x_t$ at the state level. We use county-level parameter $\theta$ together with a county-to-county mobility graph $A$ as input.  We use the DCGRU layer \citep{li2017diffusion} to encode the graph in a GRU. We use a linear layer to map the county-level output to hidden features  at the state level. For both the state-level encoder and  decoder, we use multi-layer GRUs. 


The input $\theta_{1:t}$ is the county-level parameters for LEAM-US with a dimension of 10. The county level embedding uses 1 layer DCGRU with a width of 16. The internal state is at state level with dimension of 16. The state level encoder and decoder use 3 layer GRUs with width of 128. The dimension of the latent process $z_{1:t}$ is 32. The dimension of output $x_{1:t}$ is 24, including the incidence and prevalence for 12 compartments. We trained STNP model for $500$ steps with learning rate fixed at $10^{-3}$ using Adam optimizer. We perform early stopping with $50$ patience for both offline learning and Bayesian active learning.

\subsection{Acquisition Function}
\label{app:acquisition}
\textbf{Maximum Mean STD.} Mean STD \citep{gal2016dropout} is a heuristic used to  estimate  the model uncertainty. For each augmented parameter $\theta$, we sample multiple $z_{1:T}$  and generate a set of predictions $\{\hat{x}_{1:T}\}$. For a length $T$ sequence with dimension $D$,  we compute the standard deviation  $\sigma_{t,d}$ for time step $t$ and feature $d$. Mean STD computes  $\bar{\sigma} = \frac{1}{TD}\sum_{t=1}^T\sum_{d=1}^D \sum \sigma_{t,d}$ for each parameter $\theta$. We select the  $\theta$  with the maximum $\bar{\sigma}$.  Empirically, we found  that Mean STD often becomes over-conservative and tends to explore less. 

\textbf{Maximum Entropy.} Maximum entropy computes the maximum predictive entropy as $H(\hat{x}) = - \mathbb{E}[\log p(\hat{x}_{1:T})]$. In general,  entropy is intractable for continuous output. Our NP model implicitly  assumes the predictions follow a multivariate Gaussian, which allows us to compute the differential entropy \citep{jaynes1957information}.    We follow the same procedure as Mean STD to estimate the empirical  covariance $\Sigma \in \mathbb{R}^{TD\times TD}$and compute the differential entropy for each parameter as $H = \frac{1}{2}\ln{|\Sigma|}+\frac{TD}{2}(1+\ln{2\pi})$. We select the parameter $\theta$  with the maximum entropy. 

\subsection{Implementation Details.}
\label{app:implementation}
For both GP and \ours{} model mimicking SEIR simulation, we ran experiments using CPU. No GPU accelerator is needed for this simple model. It takes 5 hours to converge.  For \ours{} model mimicking LEAM-US simulation, we ran experiments with GEFORCE RTX 2080. It takes one day for the training to converge. For all experiments, we run with three different random seeds. 

We implement STNP to mimic the reaction diffusion simulator with feed rate ($\theta_0$) and kill rate ($\theta_1$) as input. The initial state of the reaction is fixed. We use multiple convolutional layers  with a linear layer to encode the spatial data into latent space. We use an LSTM layer to encode the latent spatial data with $\theta_0$, $\theta_1$ to map the input-output pairs to hidden features $z_{1:5}$. With $(\theta_{0}, \theta_{1})$, and $z_{1:5}$ sampled from the posterior distribution, we use an LSTM layer and deconvolutional layers to simulate reaction diffusion sequence. For each epoch, we randomly select $20\%$ samples as context sequence.

\subsection{Implementation Details for MAF,  SNL, RNN, and DMFAL.}
\label{app:implementation_offline}
We use the likelihood-free inference code \cite{durkan2020contrastive} to implement Masked Autoregressive Flow (MAF) model and Sequential Neural Likelihood (SNL) framework. We use the VAE-based deep surrogate model code for Deep Multi-fidelity Active Learning (DMFAL) \cite{li2020deep} to implement DMFAL. For DMFAL, we set the number of fidelity levels to $1$ to meet our task setting. Note that we only use DMFAL model for offline learning test as their proposed acquisition function is EIG applied to multiple fidelity levels, which is equivalent to EIG once we reduce the number of fidelity levels to $1$. We also include the RNN baseline with variational dropout, which only uses the NP decoder for surrogate modeling. We use this baseline for ablation study to show the effectiveness to include the latent processes of Neural processes model.

We follow the same hyperparameter setting used in DMFAL \citep{li2020deep} for the standard Heat simulation task. For other experiments, we tune the hyperparameters of baseline models including the learning rate and the hidden state size to optimize their performance.

\section{Additional Results}

\begin{table*}[t!]
\centering
\caption{Performance comparison of different acquisition functions in NP model for SEIR simulator}
\resizebox{\textwidth}{!}{
\begin{tabular}{c|ccccc}
\toprule
Percentage of samples & LIG & EIG & Random & MeanSTD & MaxEntropy \\
\midrule
$1.11\%$ & $\textbf{365.87}\pm142.87$ & $435.08\pm32.38$ & $480.68\pm5.24$ & $480.22\pm12.63$ & $427.73\pm61.36$\\
 \midrule
$1.85\%$ & $\textbf{236.9}\pm50.6$ & $340.27\pm30.84$ & $398.33\pm131.05$ & $314.75\pm111.42$ & $302.24\pm119.84$\\
 \midrule
$2.96\%$ & $\textbf{119.26}\pm14.22$ & $291.15\pm10.60$ & $244.27\pm148.89$ & $158.94\pm36.6$ & $186.88\pm57.48$\\
 \midrule
$4.07\%$ & $\textbf{96.73}\pm17.07$ & $261.60\pm7.78$ & $116.8\pm9.1$ & $127.36\pm27.97$ & $146.72\pm26.06$\\
\bottomrule
\end{tabular}}
\label{tb:SEIR_INP}
\end{table*}

\begin{table*}[t!]
\centering
\caption{Performance comparison of different acquisition functions in GP model and SNL model for SEIR simulator}
\resizebox{0.88\textwidth}{!}{
\begin{tabular}{c|cccc}
\toprule
Percentage of samples & Random & MeanSTD & MaxEntropy & SNL\\
\midrule
$1.11\%$ & $663.76\pm46.36$ & $606.81\pm6.89$ & $586.25\pm58.44$ & $707.61\pm44.42$\\
 \midrule
$1.85\%$ & $637.12\pm13.45$ & $619.15\pm36.42$ & $628.54\pm71.34$ & $669.03\pm73.19$\\
 \midrule
$2.96\%$ & $597.3\pm19.59$ & $589.72\pm24.9$ & $568.84\pm19.05$ & $668.67\pm72.42$\\
 \midrule
$4.07\%$ & $519.98\pm17.86$ & $530.07\pm 32.95$ & $578.34\pm68.7$ & $685.28\pm53.00$\\
\bottomrule
\end{tabular}}
\label{tb:SEIR_GP}
\end{table*}

\begin{table*}[t!]
\centering
\caption{Active learning performance comparison using MAE on Heat simulator.}
\resizebox{1.\textwidth}{!}{
\begin{tabular}{c|ccccc}
\toprule
Percentage of samples & LIG & EIG & Random & MeanSTD & MaxEntropy \\
\midrule
$5.21\%$ & \textbf{1.55e-2} $\pm$ 1.9e-3 & 1.64e-2 $\pm$ 1.9e-3 & 1.74e-2 $\pm$ 2.2e-3 & 1.77e-2 $\pm$ 2e-3 & 1.83e-2 $\pm$ 1.1e-3\\
\midrule
$7.81\%$ & \textbf{1.33e-2} $\pm$ 1.7e-3 & 1.56e-2 $\pm$ 1.1e-3 & 1.49e-2 $\pm$ 3.3e-3 & 1.60e-2 $\pm$ 3.7e-3  & 1.58e-2 $\pm$ 4e-4\\
\midrule
$10.42\%$ & \textbf{1.02e-2} $\pm$ 3.4e-3 & 1.38e-2 $\pm$ 1.6e-3 & 1.23e-2 $\pm$ 1.5e-3 & 1.30e-2 $\pm$ 2.4e-3 & 1.61e-2 $\pm$ 8e-4\\
\midrule
$13.02\%$ & \textbf{9.3e-3} $\pm$ 3.1e-3 & 1.05e-2 $\pm$ 2.2e-3 & 1.14e-2 $\pm$ 1.3e-3 & 1.27e-2 $\pm$ 2.5e-3 & 1.46e-2 $\pm$ 6e-4\\
\midrule
$15.62\%$ & \textbf{6.7e-3} $\pm$ 5e-4 & 1.08e-2 $\pm$ 1.8e-3 & 1.05e-2 $\pm$ 9e-4 & 1.17e-2  $\pm$ 1.2e-3 & 1.43e-2 $\pm$ 2e-4\\
\bottomrule
\end{tabular}}
\label{tb:heat_STNP}
\end{table*}

\begin{table*}[t!]
\centering
\caption{Performance comparison of different acquisition functions in STNP model for RD simulator}
\resizebox{\textwidth}{!}{
\begin{tabular}{c|cccccc}
\toprule
Percentage of samples & LIG & EIG & Random & MeanSTD & MaxEntropy \\
\midrule
$6.25\%$ &$\textbf{4.562}\pm0.114$&$4.861\pm0.433$&$5.325\pm0.361$&$5.264\pm0.298$&$4.826\pm0.336$\\
\midrule
$12.50\%$ &$\textbf{3.841}\pm0.253$&$4.590\pm0.529$&$4.179\pm0.045$&$4.157\pm0.252$&$4.084\pm0.042$\\
\midrule
$18.75\%$ &$\textbf{3.165}\pm0.142$&$4.162\pm0.696$&$3.602\pm0.182$&$3.675\pm0.229$&$3.694\pm0.140$\\
\midrule
$25.00\%$ &$\textbf{2.415}\pm0.083$&$3.993\pm0.847$&$3.140\pm0.165$&$3.339\pm0.111$&$3.302\pm0.284$\\
\midrule
$31.25\%$ &$\textbf{2.302}\pm0.007$&$3.714\pm0.861$&$2.561\pm0.243$&$2.791\pm0.072$&$2.912\pm0.473$\\
\bottomrule
\end{tabular}}
\label{tb:RD_STNP}
\end{table*}

\begin{table*}[t!]
\centering
\caption{Performance comparison of different acquisition functions in STNP model for LEAM-US simulator, population divided by 1000.}
\resizebox{\textwidth}{!}{
\begin{tabular}{c|cccccc}
\toprule
Percentage of samples & LIG & EIG & Random & MeanSTD & MaxEntropy \\
\midrule
$11.1\%$ &$\textbf{14.447}\pm1.087$&$19.067\pm3.981$&$20.961\pm5.548$&$35.356\pm28.706$&$65.498\pm13.324$\\
\midrule
$13.7\%$ &$\textbf{11.704}\pm0.216$&$16.372\pm3.663$&$13.418\pm0.815$&$16.092\pm3.11$&$30.496\pm24.333$\\
\midrule
$21.3\%$ &$\textbf{7.593}\pm0.822$&$11.754\pm1.713$&$9.332\pm0.601$&$11.191\pm0.184$&$10.028\pm2.065$\\
\midrule
$28.9\%$ &$\textbf{6.539}\pm0.618$&$9.455\pm0.595$&$8.077\pm0.657$&$7.908\pm0.536$&$8.417\pm0.616$\\
\midrule
$36.5\%$ &$\textbf{6.008}\pm1.079$&$8.596\pm0.741$&$6.719\pm0.383$&$7.533\pm0.861$&$7.431\pm0.776$\\
\bottomrule
\end{tabular}}
\label{tb:LEAM_STNP}
\end{table*}

\subsection{INP, GP, and SNL Model}
Table \ref{tb:SEIR_INP} and Table \ref{tb:SEIR_GP} show the average results together with the standard deviation of \ours{}, GP, and SNL model for SEIR simulator after running experiments three times. The performance of \ours{} with the proposed LIG is much better than GP and SNL baselines at each iteration.

\subsection{Active learning performance comparison.}
\label{app:active_learning}
Table \ref{tb:heat_STNP}, Table \ref{tb:RD_STNP}, and Table \ref{tb:LEAM_STNP} show the active learning performance comparison results on Heat, RD, and LEAM-US simulation task. The performance of \ours{} with the proposed LIG always outperforms the baselines at each iteration among all $3$ tasks. 


            


\subsection{Batch  Active Learning with LIG.}

\begin{figure}[t!]
  \centering
  \includegraphics[width=0.6\linewidth,trim={20 20 20 0}]{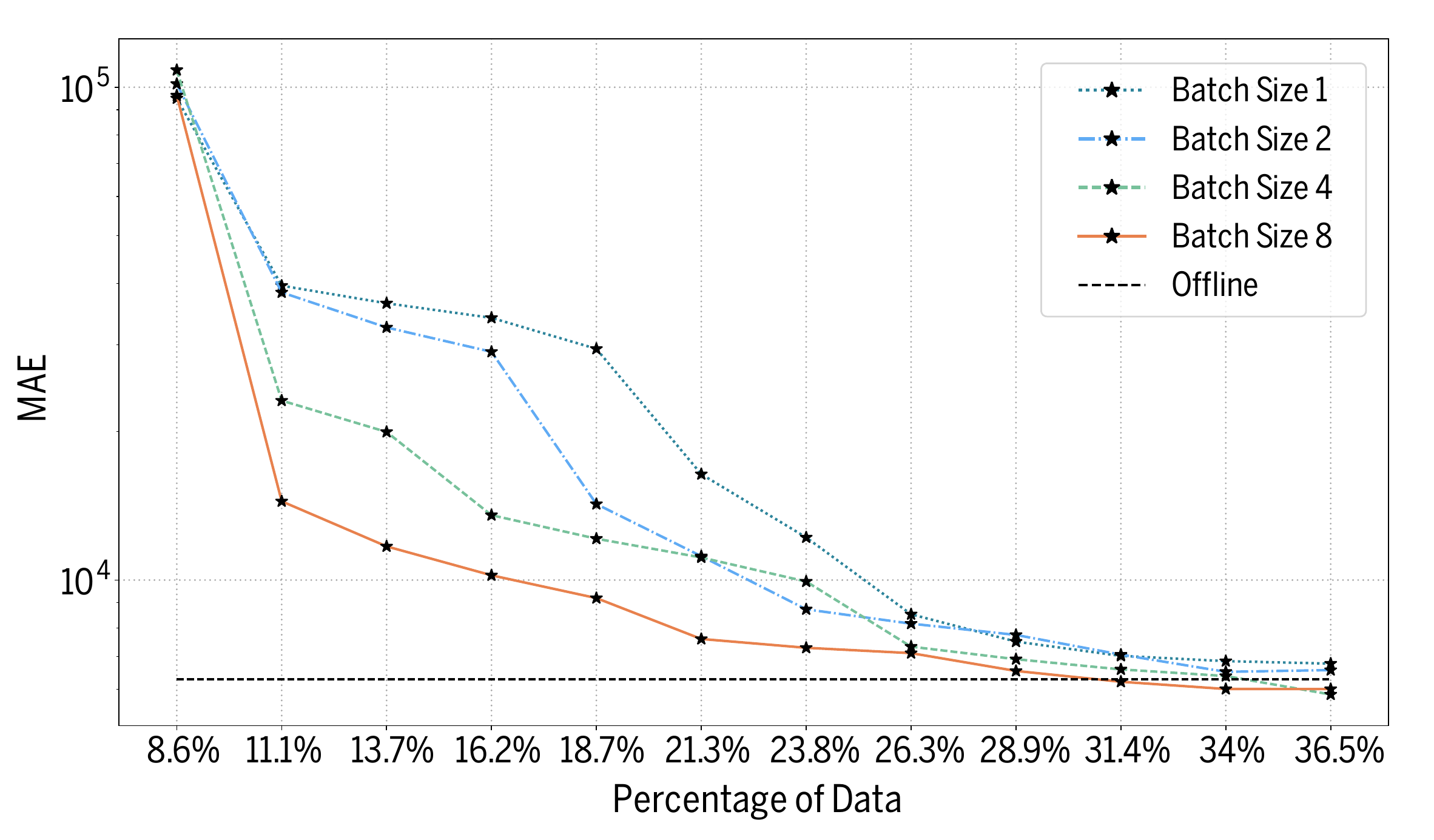}
  \caption{Batch size comparisons for LIG on the LEAM-US simulator. MAE loss versus the percentage of samples for \ours{} during Bayesian active learning.}
  \label{fig:ablation_pred}
\end{figure}
Figure \ref{fig:ablation_pred} compares 4 different setups:
8 batches (size 1), 4 batches (size 2), 2 batches
(size 4), and 1 batch (size 8).

\end{document}